\documentclass[letterpaper]{article} 
\usepackage[preprint]{aaai2027}    
\usepackage[hyphens]{url}  
\usepackage{graphicx}      
\usepackage{natbib}        
\usepackage{caption}       
\usepackage{amsmath}       
\usepackage{amssymb}       
\usepackage{mleftright}   
\usepackage{amsthm}                      
\newtheorem{proposition}{Proposition}    
\usepackage{algorithm}
\usepackage{algorithmic}

\usepackage{newfloat}
\usepackage{listings}
\DeclareCaptionStyle{ruled}{labelfont=normalfont,labelsep=colon,strut=off} 
\floatstyle{ruled}
\newfloat{listing}{tb}{lst}{}
\floatname{listing}{Listing}
\usepackage{subcaption}

\usepackage{placeins}
\usepackage{booktabs}
\usepackage{colortbl}
\definecolor{rankfirst}{gray}{0.75}
\definecolor{ranksecond}{gray}{0.85}
\definecolor{rankthird}{gray}{0.93}
\newcommand{\firstscore}[1]{\cellcolor{rankfirst}\textbf{#1}}
\newcommand{\secondscore}[1]{\cellcolor{ranksecond}\underline{#1}}
\newcommand{\thirdscore}[1]{\cellcolor{rankthird}#1}
\title{DASH: Divergence-Adaptive Supervision Horizons for On-Policy Self-Distillation of Reasoning Models}

\author{
    ZhiYan Hou\equalcontrib\textsuperscript{\rm 1,\rm 2,\rm 3},
    Xinyu Tang\equalcontrib\textsuperscript{\rm 2,\rm 3},
    Hongyan An\textsuperscript{\rm 1,\rm 2,\rm 3},
    Jianjin Zhang\textsuperscript{\rm 2,\rm 3},\\
    Weizhen Wang\textsuperscript{\rm 2,\rm 3},
    Yunyun Han\textsuperscript{\rm 2,\rm 3},
    Gengsheng Li\textsuperscript{\rm 1},
    Xiangzhao Hao\textsuperscript{\rm 1},\\
    Haiyun Guo\textsuperscript{\rm 1,\rm 4},
    Wenbin Hu\textsuperscript{\rm 6,\rm \dag},
    Jinqiao Wang\textsuperscript{\rm 1,\rm 4,\rm 5},
    Yafeng Deng\textsuperscript{\rm 2,\rm 3,\rm \dag}
}
\affiliations{
    \textsuperscript{\rm 1}Institute of Automation, Chinese Academy of Sciences\\
    \textsuperscript{\rm 2}EverMind\\
    \textsuperscript{\rm 3}Shanda Group\\
    \textsuperscript{\rm 4}University of Chinese Academy of Sciences\\
    \textsuperscript{\rm 5}Wuhan AI Research\\
    \textsuperscript{\rm 6}Wuhan University\\
    \textsuperscript{\dag}Corresponding authors.
}

\begin{document}

\maketitle

\begin{abstract}

Reinforcement learning with verifiable rewards (RLVR) improves the reasoning capabilities of large language models using automatically verifiable outcome signals, but these signals are typically sparse and at the sequence-level. On-policy self-distillation (OPSD) mitigates this sparsity by querying a privileged teacher at student-visited prefixes and providing dense token-level distributional supervision. Although this dense supervision alleviates signal sparsity, we find that standard OPSD still underexploits the temporal structure of the rollout. It assigns every local divergence the same coefficient, regardless of its position or the divergence sequence in which it occurs. In on-policy autoregressive generation, the same divergence magnitude can follow different discrepancy histories, reflecting different evolutions of the mismatch between the teacher and student. Since the local scalar alone cannot distinguish these temporal contexts, standard OPSD cannot adapt its token-level weights to the realized discrepancy sequence. To address this limitation, we propose \emph{Divergence-Adaptive Supervision Horizons} (DASH). DASH maps the gap between each local distillation signal and the sequence-level mean to an adaptive propagation gate and then uses these gates to control backward multi-step aggregation. By doing so, DASH adjusts token-level supervision weights according to how local divergences evolve during generation. Experiments on three mathematical reasoning benchmarks across three model scales show that DASH improves over our matched vanilla OPSD reruns on every benchmark at all three scales. DASH reuses the teacher and student distributions that OPSD already computes, so the gains require no additional teacher or student forward pass.

\end{abstract}

\begin{links}
    \link{Code}{https://github.com/DBtxy/DASH-OPSD}
\end{links}


\section{Introduction}
Reinforcement learning with verifiable rewards (RLVR) has become an important post-training paradigm for improving the mathematical reasoning and code generation capabilities of large language models~\cite{lambert2024tulu3,shao2024deepseekmath,guo2025deepseekr1}. It optimizes models using automatically verifiable signals, such as answer correctness and program execution results~\cite{cobbe2021gsm8k,hendrycks2021math,chen2021codex}, avoiding costly human preference annotations~\cite{ouyang2022instructgpt,rafailov2023dpo} and scales effectively~\cite{yu2025dapo}.
However, standard RLVR typically applies the same sparse sequence-level outcome reward to all token-level decisions in a response, making it difficult to distinguish their individual contributions to the final outcome~\cite{uesato2022solving,lightman2024verify}.This temporal credit assignment challenge limits both sample efficiency and optimization stability in long-horizon reasoning~\cite{sutton1988td,kazemnejad2025vineppo,parthasarathi2025grpolambda}.

To mitigate this coarse credit assignment, recent work has introduced on-policy distillation (OPD) and on-policy self-distillation (OPSD)~\cite{agarwal2024onpolicy,gu2024minillm,zhao2026selfdistilled}, which provide dense token-level supervision on student-generated trajectories~\cite{hinton2015distilling,zelikman2022star}. In OPSD, the student generates a response, while the same model, conditioned on a reference answer or other privileged information~\cite{vapnik2009lupi,lopezpaz2016privileged}, acts as the teacher and provides distributional supervision at the student-visited prefixes.


\begin{figure*}[t]
    \centering
    \begin{subfigure}[t]{0.32\textwidth}
        \centering
        \includegraphics[width=\linewidth]{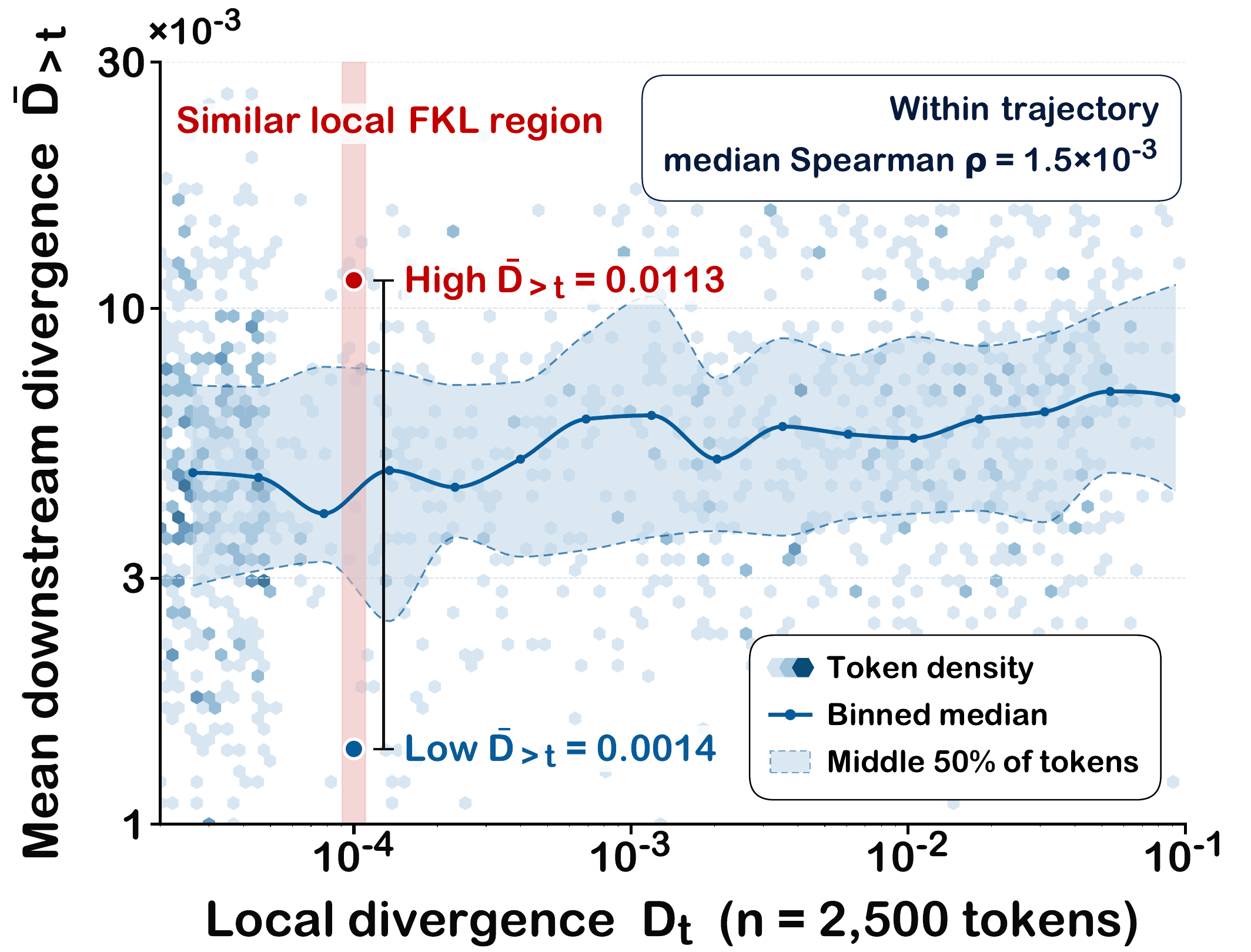}
        \caption{Weak local-to-future association.}
        \label{fig:motivation_spearman}
    \end{subfigure}
    \hfill
    \begin{subfigure}[t]{0.34\textwidth}
        \centering
        \includegraphics[width=\linewidth]{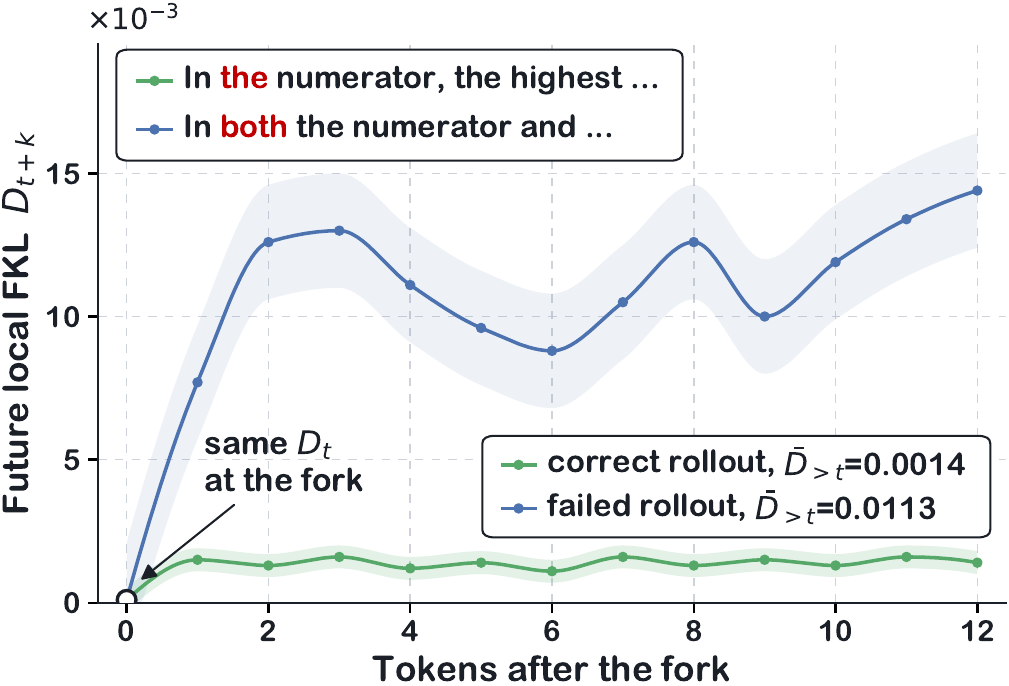}
        \caption{Matched local divergence, divergent futures.}
        \label{fig:motivation_futures}
    \end{subfigure}
    \hfill
    \begin{subfigure}[t]{0.32\textwidth}
        \centering
        \includegraphics[width=\linewidth]{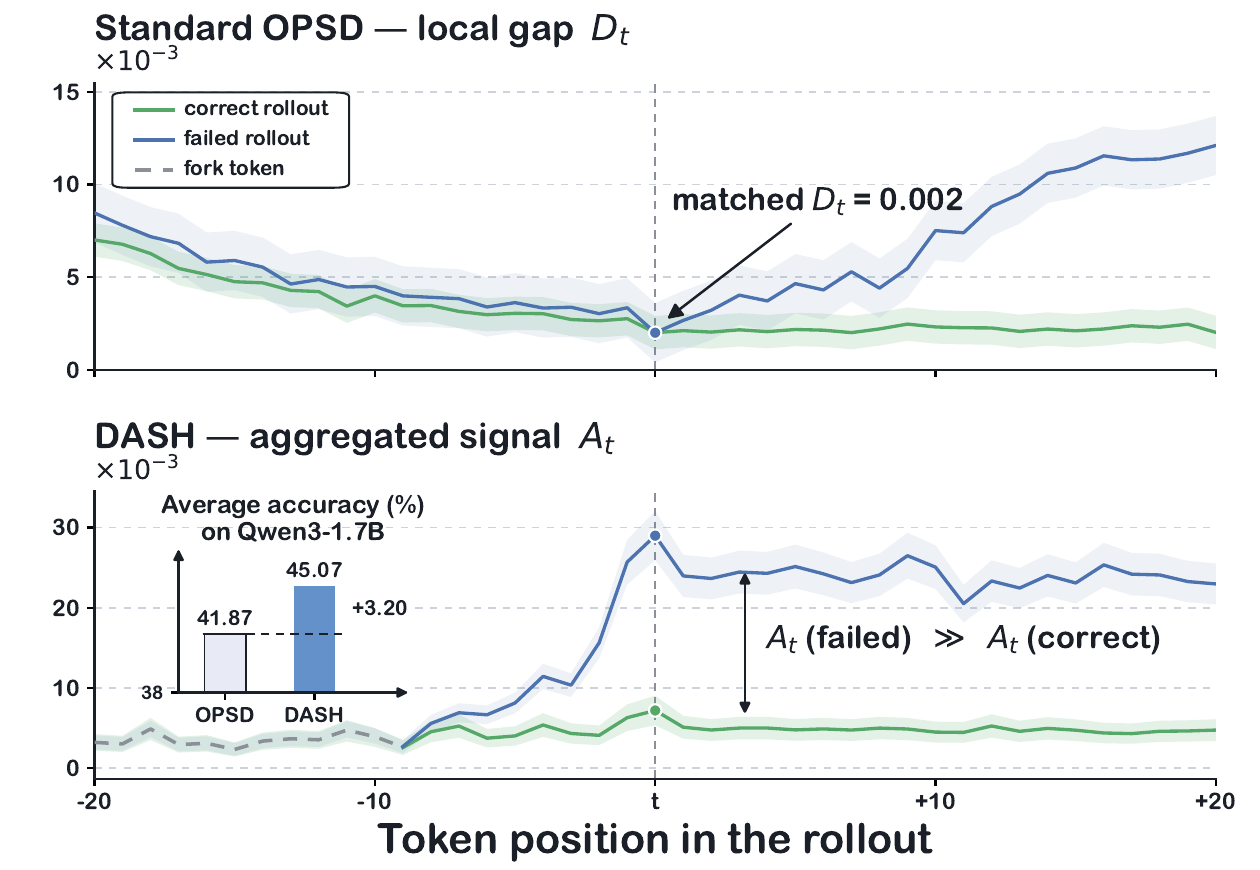}
        \caption{DASH: Sequence-aware signal allocation.}
        \label{fig:motivation_mechanism}
    \end{subfigure}

    \caption{
    \textbf{Motivation and mechanism of DASH.}
    (a) The current local divergence $D_t$ is weakly associated
    with the mean future divergence $\bar{D}_{>t}$ within
    student-generated trajectories. The plot visualizes 2,500
    of the 30k analyzed tokens.
    (b) Two rollouts matched by local divergence at a reasoning fork subsequently exhibit sharply different divergence profiles and terminal outcomes.
    (c) Vanilla OPSD aggregates the matched local divergences with the same coefficient, regardless of temporal discrepancy profiles. DASH instead uses the realized discrepancy sequence to produce path-dependent aggregated supervision signals.}
    \label{fig:motivation}
\end{figure*}


However, dense token-level supervision primarily alleviates signal sparsity, while leaving the temporal structure of the rollout underexploited. This is because vanilla OPSD aggregates position-wise distributional discrepancies using uniform coefficients that are independent of both token position and the temporal discrepancy profile along the rollout. But in on-policy autoregressive generation~\cite{ross2011dagger,bengio2015scheduled}, the same local discrepancy value can arise after different discrepancy histories that reflect how the mismatch between teacher and student has evolved over student-visited prefixes. The local scalar alone cannot distinguish these temporal contexts, yet vanilla OPSD assigns the same explicit aggregation coefficient in every case. Consequently, vanilla OPSD cannot adapt token-level distillation weights to the temporal evolution of the discrepancy sequence.

To make coefficient allocation sensitive to the temporal discrepancy profile, we propose \emph{Divergence-Adaptive Supervision Horizons} (DASH), a divergence-adaptive multi-step aggregation method for OPSD. At each position, DASH converts the gap between the local distillation signal and its sequence-level mean into an adaptive propagation gate $\lambda_t$. These gates control a backward multi-step recursion, yielding a dynamic coefficient for each local loss that is determined by its preceding gate path. DASH thus replaces the uniform coefficient profile of vanilla OPSD with a path-dependent sequence-level objective whose effective supervision horizon adapts to the realized discrepancy sequence.

Our contributions are threefold. \textbf{Firstly}, we identify a temporal coefficient allocation gap in vanilla OPSD: assigning the same weight to every local divergence makes the objective invariant to different temporal arrangements of the same divergence values, preventing supervision weights from adapting to temporal structure (Section~\ref{sec:order_analysis}). \textbf{Secondly}, we propose \emph{Divergence-Adaptive Supervision Horizons} (DASH), which uses the gap between each local divergence and the sequence-level mean to adapt the effective supervision horizon and produces temporally conditioned token-level weights through multi-step aggregation, enabling dense distillation supervision to exploit temporal structure (Section~\ref{sec:dash}). \textbf{Thirdly}, with negligible additional training overhead, DASH obtains the highest overall scores among the compared results across three mathematical reasoning benchmarks and three model scales (Table~\ref{tab:main_results}), supported by comprehensive ablation studies (Sections~\ref{subsec:component_ablation} and~\ref{subsec:parameter_ablations}).



\section{Related Work}

\subsection{Reinforcement Learning with Verifiable Rewards}

Reinforcement learning with verifiable rewards (RLVR) improves language-model reasoning using automatically checked signals, such as answer matching, program execution, and unit tests~\cite{lambert2024tulu3,cobbe2021gsm8k,chen2021codex}. DeepSeekMath introduced Group Relative Policy Optimization (GRPO), while DeepSeek-R1 demonstrated the effectiveness of large-scale RLVR for mathematical and coding reasoning~\cite{shao2024deepseekmath,guo2025deepseekr1}. Standard RLVR nevertheless relies primarily on sparse sequence-level rewards~\cite{schulman2017ppo,yu2025dapo,yue2025rlvr}, providing limited direct supervision for intermediate decisions in long reasoning trajectories~\cite{wei2022cot,uesato2022solving}. Process supervision introduces step-level feedback through process reward models~\cite{lightman2024verify,uesato2022solving,wang2024mathshepherd,setlur2025pav}, whereas temporal credit-assignment methods such as GRPO-$\lambda$ redistribute outcome-level signals across sampled decisions using eligibility traces and critic-free temporal-difference estimates~\cite{parthasarathi2025grpolambda,sutton1988td,schulman2016gae,kazemnejad2025vineppo}. These approaches address learning from sparse reward signals; Our work investigates how dense, differentiable token-level supervision in OPSD should be allocated across a rollout.

\subsection{On-Policy Self-Distillation}

On-policy self-distillation (OPSD) trains a model on its own rollouts by using the same model under different contexts as an unprivileged student and a privileged teacher~\cite{agarwal2024onpolicy,gu2024minillm,zelikman2022star}. The teacher conditions on reference solutions or other privileged information and provides token-level distributional supervision at student-visited prefixes~\cite{zhao2026selfdistilled,vapnik2009lupi,lopezpaz2016privileged,pinto2018asymmetric,weihs2021imitation}. Recent OPSD methods refine either privileged supervision or token weighting~\cite{jin2026eopd,li2026phf,xie2026iwopd,lin2026renio,shen2026purified,liu2026teacherhelp}. AVSD combines cross-view consensus with view-specific signals from multiple privileged teacher contexts~\cite{nguyen2026avsd}, while PW-OPSD applies fixed position-dependent weights motivated by variation in teacher-token reliability~\cite{liu2026pwopsd}. These methods improve the supervision source or impose predefined positional weighting. DASH instead focuses on how local distillation signals are aggregated across the rollout. It uses sequence-relative divergence gaps to construct adaptive propagation gates and temporally conditioned token-level coefficients through multi-step aggregation, allowing dense supervision to exploit the temporal evolution of local discrepancies.

\section{Preliminaries}
\label{sec:preliminaries}

Let $\pi_\theta$ denote the student policy. Given a problem $x$, the student samples a response $\mathbf y=(y_1,\ldots,y_T)\sim P_\theta(\cdot\mid x)$. We denote the student-generated prefix at position $t$ by $s_t=(x,y_{<t})$.

On-policy self-distillation (OPSD)~\cite{agarwal2024onpolicy,zhao2026selfdistilled} uses a single model as both the student and a privileged teacher under different conditioning contexts. The student generates the rollout from the problem $x$, while the teacher additionally conditions on privileged information $z$, such as a reference solution. At each student-visited prefix $s_t$, their distributions are

\begin{equation}
    \pi_t^{\mathrm{S}}
    =
    \pi_\theta(\cdot\mid s_t),
    \qquad
    \pi_t^{\mathrm{T}}
    =
    \operatorname{sg}
    \left[
        \pi_{\bar\theta}(\cdot\mid s_t,z)
    \right],
    \label{eq:student_teacher_distributions}
\end{equation}

where $\bar\theta$ denotes the teacher parameters and $\operatorname{sg}[\cdot]$ blocks gradients through the teacher target. Querying the privileged teacher along the student rollout provides dense token-level distributional supervision~\cite{agarwal2024onpolicy,liu2026teacherhelp}.

Our main setting uses the forward KL divergence from the privileged teacher to the student as the local distillation loss~\cite{hinton2015distilling,kim2016seqkd}:

\begin{equation}
    \begin{split}
        d_t
        &=
        D_{\mathrm{KL}}
        \left(
            \pi_t^T
            \,\middle\|\,
            \pi_t^S
        \right) \\
        &=
        \sum_{v\in\mathcal A}
        \pi_t^T(v)
        \log
        \frac{\pi_t^T(v)}{\pi_t^S(v)}.
    \end{split}
    \label{eq:local_distillation_loss}
\end{equation}
\begin{figure*}[t]
    \centering
    \includegraphics[width=0.98\textwidth]
    {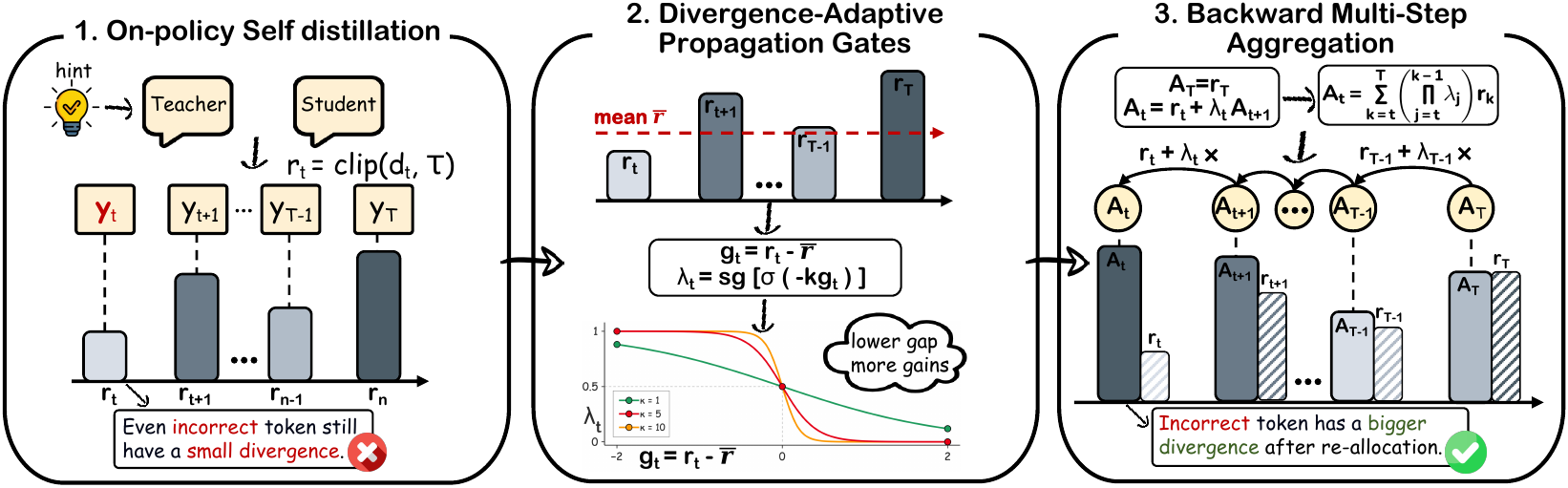}
    \caption{Overview of DASH. A privileged teacher evaluates a
    student-generated rollout to produce local distillation signals.
    DASH converts their sequence-relative gaps into adaptive
    propagation gates and applies backward multi-step aggregation to
    construct temporally conditioned token-level weights and adaptive
    supervision horizons.}
    \label{fig:method_overview}
\end{figure*}

For a sampled response $\mathbf y$, vanilla OPSD uniformly aggregates the local distillation losses:
\begin{equation}
    \mathcal L_{\mathrm{OPSD}}(\theta;\mathbf y)
    =
    \frac{1}{T}
    \sum_{t=1}^{T}d_t.
    \label{eq:opsd_loss}
\end{equation}
Vanilla OPSD treats the sampled response as fixed during optimization
and differentiates Eq.~\eqref{eq:opsd_loss} only through the student
distributions~\cite{agarwal2024onpolicy,liu2026teacherhelp}.

\section{Method}
\label{sec:method}

Vanilla OPSD assigns the same explicit coefficient to every local
distillation loss, regardless of its position or the surrounding
evolution of discrepancies along the rollout. We first characterize
this limitation and use the exact gradient of the corresponding
expected on-policy objective as a structural reference. Beyond the
direct distillation term retained by standard training, the exact
gradient contains a trajectory term whose coefficient depends on
subsequent divergences. Motivated by this order-dependent structure,
we introduce DASH, which uses sequence-relative divergence gaps,
adaptive propagation gates, and backward multi-step aggregation to
construct temporally conditioned coefficients for direct distillation
gradients. Figure~\ref{fig:method_overview} provides an overview.

\subsection{Coefficient Structure of Vanilla OPSD}
\label{sec:order_analysis}

For a sampled rollout $\mathbf y$, vanilla OPSD retains the detached gradient
\begin{equation}
    \nabla_\theta
    \mathcal L_{\mathrm{OPSD}}(\theta;\mathbf y)
    =
    \frac{1}{T}
    \sum_{t=1}^{T}
    \nabla_\theta^{\mathrm{loc}} d_t.
    \label{eq:detached_opsd_gradient}
\end{equation}
Thus, every local gradient has the same explicit coefficient $1/T$, independent of the ordered discrepancy profile $\mathbf d=(d_1,\ldots,d_T)$. This does not make the local gradients sequence-independent: each $\nabla_\theta^{\mathrm{loc}}d_t$ is still evaluated at its corresponding student-generated prefix $s_t$. Only the explicit coefficient profile is independent of how the local discrepancies evolve along the rollout.

To isolate coefficient structure, consider the following fixed-horizon
surrogate. Let $H$ be a deterministic truncation horizon; after EOS,
the trajectory enters a parameter-independent absorbing state and
$d_t=0$ thereafter.
\begin{equation}
    \mathcal J_H(\theta)
    =
    \mathbb E_{(x,z)\sim\mathcal D}
    \mathbb E_{\mathbf y\sim P_\theta(\cdot\mid x)}
    \left[
        \frac{1}{H}
        \sum_{t=1}^{H}d_t
    \right].
    \label{eq:fixed_horizon_objective}
\end{equation}

\begin{proposition}[Fixed-Horizon Gradient Decomposition]
\label{prop:exact_gradient_decomposition}
Under the fixed-horizon construction above, the gradient of
Eq.~\eqref{eq:fixed_horizon_objective} is
\begin{equation}
\begin{split}
    \nabla_\theta\mathcal J_H(\theta)
    =
    \mathbb E\Bigg[
        &\frac{1}{H}
        \sum_{t=1}^{H}
        \nabla_\theta^{\mathrm{loc}}d_t
        \\
        &+
        \frac{1}{H}
        \sum_{u=1}^{H}
        G^D_{u+1}
        \nabla_\theta
        \log\pi_\theta(y_u\mid s_u)
    \Bigg],
    \label{eq:exact_opsd_gradient}
\end{split}
\end{equation}
where
$G^D_{u+1}=\sum_{t=u+1}^{H}d_t$
is the future divergence-to-go after position $u$.
\end{proposition}

\begin{proof}
Let
\(
D_{\mathrm{tot}}(\mathbf y;\theta)
=
\sum_{t=1}^{H}d_t
\).
Applying the log-derivative identity
\cite{williams1992reinforce,sutton1999policygradient}
to the trajectory expectation gives
\begin{equation}
\begin{split}
    \nabla_\theta\mathcal J_H(\theta)
    =
    \frac{1}{H}\,
    \mathbb E\Bigg[
        &\sum_{t=1}^{H}
        \nabla_\theta^{\mathrm{loc}}d_t
        \\
        &+
        D_{\mathrm{tot}}
        \sum_{u=1}^{H}
        \nabla_\theta
        \log\pi_\theta(y_u\mid s_u)
    \Bigg].
    \label{eq:gradient_log_derivative}
\end{split}
\end{equation}

For $t\leq u$, the value of $d_t$ is determined before $y_u$
is sampled and is therefore measurable with respect to $s_u$.
Using the zero conditional mean of the score function,
\begin{equation}
\mathbb E_{y_u\sim\pi_\theta(\cdot\mid s_u)}
\left[
    d_t
    \nabla_\theta\log\pi_\theta(y_u\mid s_u)
    \,\middle|\,
    s_u
\right]
=
0,
\qquad t\leq u.
\label{eq:causal_score_cancellation}
\end{equation}
Hence, the score term at position $u$ retains only losses at
positions $t>u$, whose sum is
$G^D_{u+1}=\sum_{t=u+1}^{H}d_t$.
Substituting this result into
Eq.~\eqref{eq:gradient_log_derivative}
yields Eq.~\eqref{eq:exact_opsd_gradient}.
\end{proof}

The first term is the direct distillation gradient retained by vanilla
OPSD. The second is a trajectory score-function term with
future-divergence coefficients, reflecting changes in future visited
prefixes when the sampling path is differentiated
\cite{ross2011dagger,bengio2015scheduled}. This fixed-horizon
decomposition provides only structural motivation: DASH does not
estimate $G^D_{u+1}$, introduce score-function gradients, or perform
future-to-past credit assignment.

\subsection{Divergence-Adaptive Supervision Horizons}
\label{sec:dash}

The fixed-horizon decomposition highlights a structural contrast:
order dependence appears in the trajectory coefficients, whereas the
direct distillation term retained by vanilla OPSD uses constant
explicit coefficients. Motivated by this contrast, DASH makes the
coefficients of direct local distillation losses depend on the realized
discrepancy sequence through adaptive supervision horizons.

Using the propagation gates, DASH applies the backward recursion~\cite{sutton1988td,schulman2016gae}

\begin{equation}
    A_T=r_T,
    \qquad
    A_t=r_t+\lambda_tA_{t+1},
    \quad
    t=T-1,\ldots,1,
    \label{eq:dash_recursion}
\end{equation}
and minimizes
\begin{equation}
    \mathcal L_{\mathrm{DASH}}
    =
    \frac{1}{T}\sum_{t=1}^{T}A_t.
    \label{eq:dash_objective}
\end{equation}
Here, $\lambda_t$ controls how strongly later local signals propagate
across the boundary between positions $t$ and $t+1$~\cite{hochreiter1997lstm,cho2014gru}. In particular, for
$k>t$, the contribution of $r_k$ to $A_t$ is scaled by
$\prod_{j=t}^{k-1}\lambda_j$. Summing these contributions over all
aggregation starting positions gives the weighted form
\begin{equation}
    \mathcal L_{\mathrm{DASH}}
    =
    \frac{1}{T}\sum_{k=1}^{T}c_kr_k,
    \qquad
    c_1=1,
    \qquad
    c_k=1+\lambda_{k-1}c_{k-1}.
    \label{eq:dash_weighted_objective}
\end{equation}
Thus, $c_k$ aggregates the contributions to $r_k$ from all starting
positions up to $k$, weighted by the intervening gates. Although its
recursion involves $\lambda_1,\ldots,\lambda_{k-1}$, each gate is
computed relative to the sequence-level mean $\bar r$. Consequently,
$c_k$ is conditioned on the realized discrepancy sequence.

Let $\ell_{t,v}$ denote the vocabulary-level summand in Eq.~\eqref{eq:local_distillation_loss}. Algorithm~\ref{alg:dash} summarizes the additional DASH computation after the local KL contributions have been obtained along a rollout.

\begin{algorithm}[t]
\caption{DASH aggregation for one student rollout. Gradients flow only
through the local signals $r_t$.}
\label{alg:dash}
\small
\textbf{Input}: Local KL contributions $\{\ell_{t,v}\}$, clipping
threshold $\tau$, and propagation sensitivity $\kappa$\par
\begin{algorithmic}[1]
\FOR{$t=1,\ldots,T$}
    \STATE
    $r_t
    \leftarrow
    \sum_{v\in\mathcal A}
    \min(\ell_{t,v},\tau)$
\ENDFOR

\STATE
$\bar r
\leftarrow
\frac{1}{T}\sum_{t=1}^{T}r_t$

\FOR{$t=1,\ldots,T-1$}
    \STATE
    $\lambda_t
    \leftarrow
    \operatorname{sg}
    \left[
        \sigma\left(-\kappa(r_t-\bar r)\right)
    \right]$
\ENDFOR

\STATE $A_T\leftarrow r_T$
\FOR{$t=T-1,T-2,\ldots,1$}
    \STATE
    $A_t\leftarrow r_t+\lambda_tA_{t+1}$
\ENDFOR

\STATE \textbf{return}
$\frac{1}{T}\sum_{t=1}^{T}A_t$
\end{algorithmic}
\end{algorithm}





\section{Experiments}

\subsection{Experimental Setup}
\label{sec:experimental_setup}

\paragraph{Models and Datasets.}
We conduct experiments with the Qwen3-1.7B, Qwen3-4B, and Qwen3-8B Instruct models~\cite{yang2025qwen3}. Our reruns use OpenThoughts-Math-30K~\cite{guha2025openthoughts}, which contains 29,434 mathematical reasoning examples, each with a problem, reference solution, and final answer. In every on-policy distillation run, the student receives only the problem and generates its own rollout. For OPSD, EOPD, PW-OPSD, and DASH, the matched privileged teacher additionally receives the reference solution, including the final answer; AVSD retains its method-defining multi-view privileged contexts. We evaluate on AIME 2024~\cite{aime2024}, AIME 2025~\cite{aime2025}, and HMMT February 2025~\cite{hmmt2025feb}, each containing 30 problems, and determine correctness using the official answer keys. For each benchmark, Avg@12 averages the binary correctness of 12 independently sampled responses per problem and then averages over all problems. We additionally report the unweighted mean across the three benchmarks.

\paragraph{Baselines.}

We compare DASH with seven baselines. \emph{Base} denotes the original model without additional training. \emph{SFT} performs supervised fine-tuning on the reference solutions~\cite{ouyang2022instructgpt}. \emph{GRPO} optimizes verifiable rewards based on final-answer correctness~\cite{shao2024deepseekmath}. \emph{OPSD} uniformly aggregates token-level distillation losses along student-generated trajectories~\cite{zhao2026selfdistilled}. \emph{EOPD} adapts token-level distillation using teacher entropy~\cite{jin2026eopd}. \emph{AVSD} combines supervision from multiple privileged teacher views~\cite{nguyen2026avsd}. \emph{PW-OPSD} applies predefined position-dependent weights to token-level distillation losses~\cite{liu2026pwopsd}. Table~\ref{tab:main_results} combines externally reported reference results with our reruns. Base, SFT, and GRPO are taken from Zhao et al.~\cite{zhao2026selfdistilled} and marked with $\dagger$. We rerun OPSD, EOPD, AVSD, PW-OPSD, and DASH. OPSD, EOPD, PW-OPSD, and DASH use the same privileged-OPSD backbone while retaining their method-specific objectives or weighting rules; AVSD retains its original multi-view teacher construction. Exact method instantiations, run counts, and result provenance are provided in Appendix~\ref{sec:reproducibility}.

\label{subsec:main_results}

\begin{table*}[t]
\centering
\small
\setlength{\tabcolsep}{1mm}
\begin{tabular}{
    @{}l
    cccc@{\hspace{0.6em}}
    cccc@{\hspace{0.6em}}
    cccc@{}
}
\toprule
& \multicolumn{4}{c}{\textbf{Qwen3-1.7B}}
& \multicolumn{4}{c}{\textbf{Qwen3-4B}}
& \multicolumn{4}{c}{\textbf{Qwen3-8B}} \\
\cmidrule(lr){2-5}
\cmidrule(lr){6-9}
\cmidrule(lr){10-13}
\textbf{Method}
& \shortstack{\textbf{AIME}\\\textbf{2024}}
& \shortstack{\textbf{AIME}\\\textbf{2025}}
& \shortstack{\textbf{HMMT}\\\textbf{2025}}
& \textbf{Average}
& \shortstack{\textbf{AIME}\\\textbf{2024}}
& \shortstack{\textbf{AIME}\\\textbf{2025}}
& \shortstack{\textbf{HMMT}\\\textbf{2025}}
& \textbf{Average}
& \shortstack{\textbf{AIME}\\\textbf{2024}}
& \shortstack{\textbf{AIME}\\\textbf{2025}}
& \shortstack{\textbf{HMMT}\\\textbf{2025}}
& \textbf{Average} \\
\midrule
Base$^{\dagger}$
& 51.50 & 36.70 & 23.10 & 37.10
& 74.90 & 66.40 & 42.20 & 61.17
& 75.80 & 65.60 & 43.90 & 61.77 \\
SFT$^{\dagger}$
& 48.40 & 36.30 & 22.70 & 35.80
& 70.20 & 62.30 & 43.40 & 58.63
& 72.30 & 64.20 & 42.90 & 59.80 \\
GRPO$^{\dagger}$
& 51.10 & 38.30 & 23.70 & 37.70
& 75.60 & \thirdscore{68.10} & \thirdscore{44.40} & 62.70
& \thirdscore{76.40} & 68.90 & \thirdscore{46.70} & 64.00 \\
\midrule
OPSD
& \thirdscore{55.60} & \thirdscore{40.80} & \thirdscore{29.20} & \thirdscore{41.87}
& \secondscore{76.40} & \secondscore{68.30} & \secondscore{46.10} & \secondscore{63.60}
& \secondscore{77.80} & \secondscore{70.80} & 45.80 & 64.80 \\
EOPD
& 51.90 & 38.10 & 26.40 & 38.80
& 73.90 & 64.70 & 41.90 & 60.17
& \secondscore{77.80} & \secondscore{70.80} & 46.40 & \thirdscore{65.00} \\
AVSD
& 55.30 & 37.50 & 24.70 & 39.17
& \thirdscore{76.20} & \secondscore{68.30} & 44.20 & \thirdscore{62.90}
& 75.40 & \thirdscore{69.60} & \secondscore{47.10} & 64.03 \\
			
PW-OPSD
& \secondscore{57.80} & \secondscore{42.20} & \secondscore{30.30} & \secondscore{43.43}
& \thirdscore{76.20} & 67.78 & 43.33 & 62.44
& \secondscore{77.80} & \secondscore{70.80} & \secondscore{47.10} & \secondscore{65.23} \\
\midrule
\textbf{DASH (Ours)}
& \firstscore{58.30} & \firstscore{45.80} & \firstscore{31.10} & \firstscore{45.07}
& \firstscore{77.20} & \firstscore{71.10} & \firstscore{46.70} & \firstscore{65.00}
& \firstscore{78.90} & \firstscore{71.40} & \firstscore{48.90} & \firstscore{66.40} \\
\bottomrule
\end{tabular}
\caption{Main results across three model scales. Scores are Avg@12 and
Average is the unweighted mean across benchmarks. $\dagger$ denotes
results reported by Zhao et al.~\cite{zhao2026selfdistilled}; all other
entries are our reruns. OPSD and DASH report four-seed means. In each column,
the top-three distinct compared results are highlighted:
\colorbox{rankfirst}{\textbf{1st}},
\colorbox{ranksecond}{\underline{2nd}}, and
\colorbox{rankthird}{3rd}; tied scores share the same rank.}
\label{tab:main_results}
\end{table*}

\paragraph{Implementation Details.}
All methods rerun by us start from the corresponding Qwen3 Instruct checkpoint, train for the full 200-step budget, and use the same checkpoint-selection and benchmark-evaluation procedures. We save checkpoints every 20 steps and, for each method, model scale, and training seed, select the checkpoint within 200 steps that maximizes the unweighted average across the three benchmarks; all three benchmark scores are taken from that single checkpoint. We refer to this as best-within-200-step reporting rather than held-out validation selection.

Following the short-rollout OPSD setting studied by Zhao et al.~\cite{zhao2026selfdistilled}, the matched privileged-OPSD runs use a maximum student completion length of 1,024 tokens. We additionally evaluate maximum lengths of 2,048 and 4,096 tokens; because they yield comparable performance, we use 1,024 tokens for computational efficiency. Unless otherwise specified, these runs use LoRA~\cite{hu2022lora} with rank 64 and scaling factor 128, a learning rate of $5\times10^{-6}$, and a global batch size of 64. Student rollouts are sampled with temperature 1.1, top-$p$ 0.95, and top-$k$ 20. DASH uses full-vocabulary forward KL as its local distillation loss, caps each vocabulary-level divergence contribution at $\tau=0.05$, and sets $\kappa=5$. For Avg@12 evaluation, we independently sample 12 responses per problem in thinking mode using temperature 1.0, top-$p$ 1.0, and a maximum of 38,912 newly generated tokens. OPSD, DASH, and every DASH ablation configuration use training seeds $s\in\{0,1,2,3\}$ and report four-seed means. Complete provenance, method-specific settings, checkpoint selection, prompts, run counts, and seed statistics are provided in Appendices~\ref{sec:reproducibility} and~\ref{sec:four_seed_results}.

\subsection{Main Results}

Table~\ref{tab:main_results} compares DASH with standard training baselines and recent on-policy distillation methods across three model scales and three mathematical reasoning benchmarks.

\noindent\textbf{Overall performance.}
DASH obtains the highest score among the compared results in all nine benchmark--model settings and the highest macro-average at each model scale. Under the matched rerun protocol, DASH raises the four-seed OPSD average from 41.87 to 45.07 on Qwen3-1.7B, from 63.60 to 65.00 on Qwen3-4B, and from 64.80 to 66.40 on Qwen3-8B. The corresponding gains are 3.20, 1.40, and 1.60 points. DASH also improves over the matched OPSD rerun on every benchmark at every model scale, indicating that its gains are not confined to a particular model capacity or dataset.

\noindent\textbf{Comparison with recent distillation methods.}
Among the displayed comparisons, PW-OPSD provides the highest competing macro-average at the 1.7B and 8B scales, whereas OPSD provides the highest at the 4B scale. DASH exceeds the compared PW-OPSD result by 1.64 points on Qwen3-1.7B and by 1.17 points on Qwen3-8B, and it exceeds the matched OPSD result by 1.40 points on Qwen3-4B. Thus, DASH attains the best displayed scores among the compared methods, without implying a statistical significance test against the single-run baselines.



\subsection{Component Ablation}
\label{subsec:component_ablation}

Having established the overall effectiveness of DASH, we next examine which aspects of its adaptive coefficient allocation account for the observed gains. Unless otherwise specified, all experiments in this subsection use Qwen3-1.7B, the same four training seeds, and the same training, checkpoint-selection, and evaluation protocol as the matched OPSD comparison. We investigate three possible sources of improvement: whether sequence-adaptive propagation is more effective than fixed multi-step aggregation, whether the direction of the gap-to-gate mapping matters, and whether the gains can be explained by the increase in average coefficient scale.

\begin{figure*}[t]
    \centering
    \begin{subfigure}[t]{0.485\textwidth}
        \centering
        \includegraphics[width=\linewidth]
        {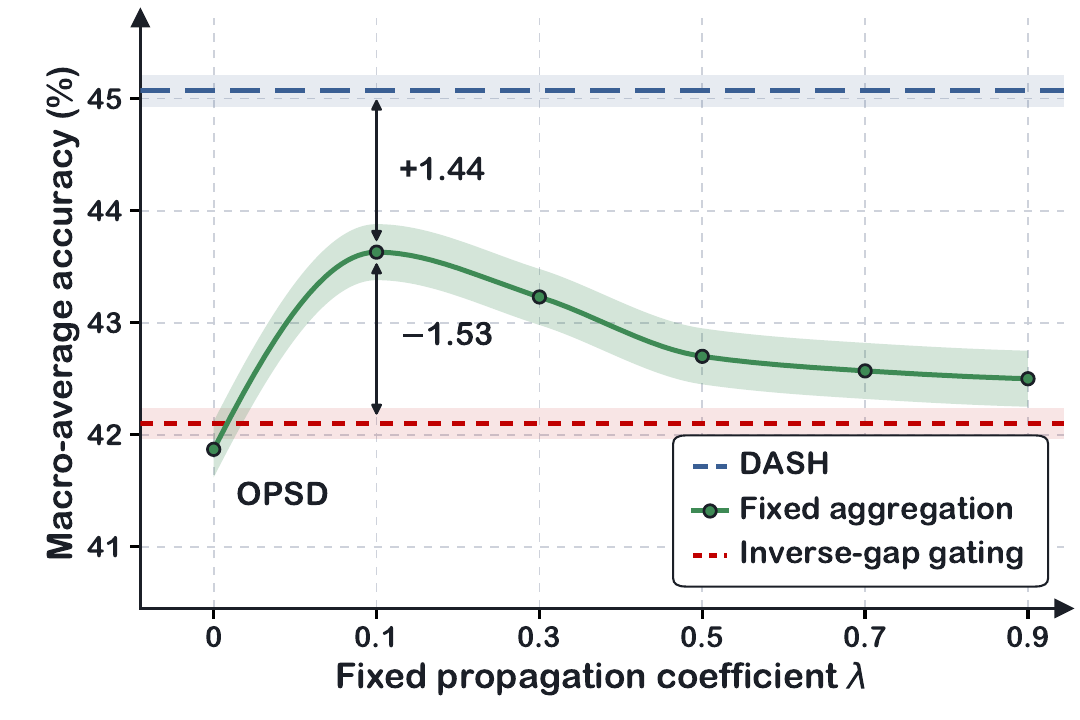}
        \caption{Fixed values of $\lambda$ and Inverse-gap.}
        \label{fig:aggregation_adaptivity}
    \end{subfigure}
    \hfill
    \begin{subfigure}[t]{0.485\textwidth}
        \centering
        \includegraphics[width=\linewidth]
        {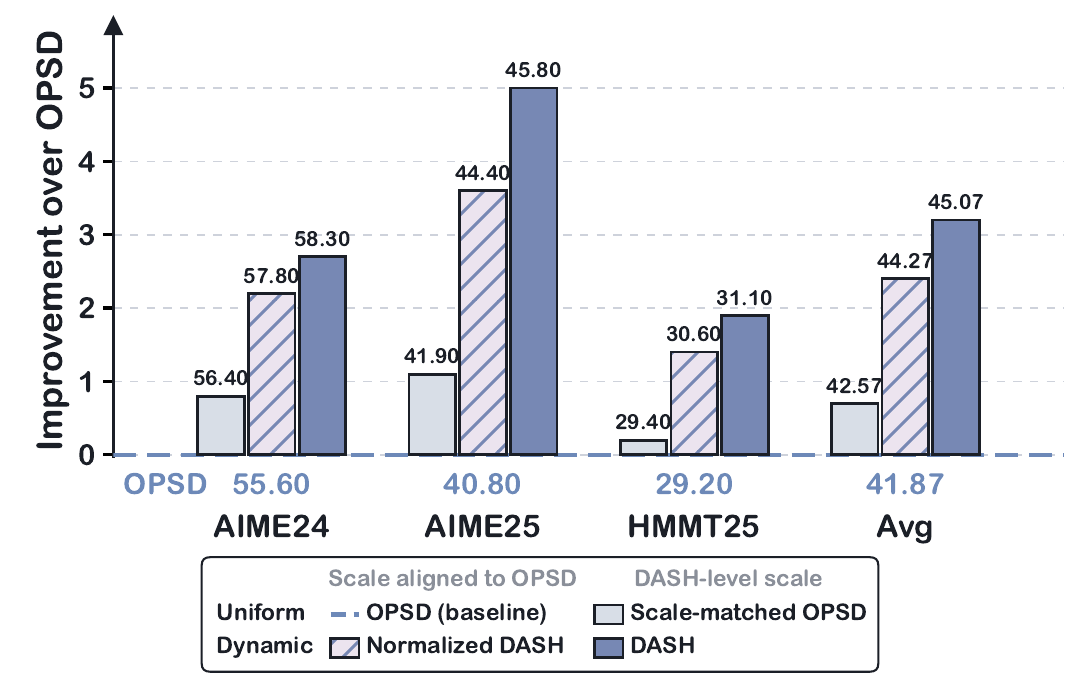}
        \caption{Coefficient allocation and average scale.}
        \label{fig:factorial_ablation}
    \end{subfigure}

    \caption{Component ablations of DASH on Qwen3-1.7B.
    (a) Fixed values of the propagation coefficient $\lambda$ test
    adaptive against fixed aggregation, while Inverse-gap reverses the
    sign of the DASH gate function. (b) Scale-matched OPSD and normalized
    DASH separate the relative coefficient profile from its average
    scale. Panel (a) reports the unweighted mean across the three
    benchmarks. Panel (b) plots improvements over the matched OPSD
    rerun, with annotations showing absolute Avg@12 scores. All values
    are means over training seeds $0,1,2,3$.}
    \label{fig:component_ablation}
\end{figure*}

\paragraph{Adaptive versus fixed propagation.}
We first examine whether adapting the propagation gates to the discrepancy sequence provides benefits beyond multi-step aggregation alone. We replace the token-wise gates $\lambda_t$ with a fixed coefficient $\lambda$ shared across all positions and rollouts. These variants retain the same backward aggregation structure as DASH but remove its dependence on the realized discrepancy sequence. We evaluate $\lambda\in\{0.1,0.3,0.5,0.7,0.9\}$, with $\lambda=0$ recovering vanilla OPSD.

As shown in Figure~\ref{fig:aggregation_adaptivity}, every tested fixed coefficient improves over vanilla OPSD, and the best setting, $\lambda=0.1$, raises the macro-average from 41.87 to 43.63. This indicates that extending supervision beyond independent local losses is beneficial even with a fixed effective horizon. DASH further improves the macro-average to 45.07, outperforming the best fixed setting by 1.44 points. Thus, fixed multi-step aggregation accounts for part of the
gain, while adapting the propagation gates to the realized discrepancy
sequence provides an additional improvement.

The fixed-$\lambda$ comparison shows that adaptive propagation is
beneficial, but it does not establish how the sequence-relative gap
should control the gate. Let $g_t=r_t-\bar r$. DASH and Inverse-gap use
the same gap, sensitivity, and aggregation recursion, differing only in
the sign inside the gate function:
\begin{equation}
    \lambda_t^{\mathrm{DASH}}
    =
    \operatorname{sg}\!\left[\sigma(-\kappa g_t)\right],
    \qquad
    \lambda_t^{\mathrm{Inv}}
    =
    \operatorname{sg}\!\left[\sigma(+\kappa g_t)\right].
    \label{eq:inverse_gap_control}
\end{equation}
Since $\lambda_t$ controls whether subsequent local signals cross the
boundary after position $t$, DASH assigns a larger gate when
$g_t<0$ and a smaller gate when $g_t>0$. It therefore produces a longer
effective supervision horizon after a below-average local signal and a
shorter horizon after an above-average one. Inverse-gap reverses exactly
this behavior. This comparison directly tests our design hypothesis
that propagation should remain more open after relatively small
discrepancies and become more restricted after unusually large ones.

As shown in Figure~\ref{fig:aggregation_adaptivity}, Inverse-gap reaches
a macro-average of 42.10. It is only 0.23 points above vanilla OPSD,
1.53 points below the best fixed coefficient, and 2.97 points below
DASH. Thus, the gain does not arise from introducing an arbitrary
gap-conditioned gate. The result supports the sign used by DASH among
the two tested adaptive mappings.

\paragraph{Coefficient allocation versus average scale.}
Finally, we examine whether DASH benefits from its relative coefficient
profile or merely from stronger overall supervision. In the weighted
objective
$\mathcal L_{\mathrm{DASH}}=T^{-1}\sum_k c_k r_k$, the recursive
aggregation produces $c_k\geq1$. DASH therefore changes both the
relative allocation of token-level supervision and the average
coefficient scale. This introduces an alternative explanation: uniformly
increasing all coefficients might reproduce the improvement without
using discrepancy-conditioned allocation.

We disentangle relative coefficient allocation from average coefficient
scale using the $2\times2$ factorial comparison in
Figure~\ref{fig:factorial_ablation}. Normalized DASH preserves the
relative coefficient profile at the OPSD scale, while scale-matched
OPSD applies uniform coefficients at the DASH scale. Dynamic allocation
improves performance by 2.40 points at the OPSD scale and 2.50 points at
the DASH scale. In contrast, increasing the average scale contributes
only 0.70 points under uniform allocation and 0.80 points under dynamic
allocation. The consistent advantage at both scales shows that the
improvement primarily comes from DASH's discrepancy-conditioned
coefficient allocation, with the increased scale providing only a
smaller complementary benefit.

Overall, the component ablations show that the gains of DASH cannot be
reproduced by fixed multi-step aggregation, a reversed adaptive mapping,
or uniform coefficient scaling. Its effectiveness primarily arises from
adapting the relative coefficient allocation to the discrepancy sequence
using the proposed gap-to-gate direction.

\subsection{Design Choices and Sensitivity}
\label{subsec:parameter_ablations}

We examine three design choices that govern how DASH constructs and
propagates token-level supervision. The distillation divergence and
vocabulary support determine the local signal used for both direct
distillation and gate construction, while the propagation sensitivity
controls how this signal shapes the adaptive supervision horizon.
These analyses assess the robustness of DASH to its main configuration
choices and support the settings used in the main experiments. Unless
otherwise specified, all experiments use Qwen3-1.7B, the same four
training seeds, and the same training, checkpoint-selection, and
evaluation protocol as the matched OPSD comparison.

\begin{figure}[t]
    \centering
    \begin{subfigure}[t]{0.55\linewidth}
        \centering
        \includegraphics[width=\linewidth]
        {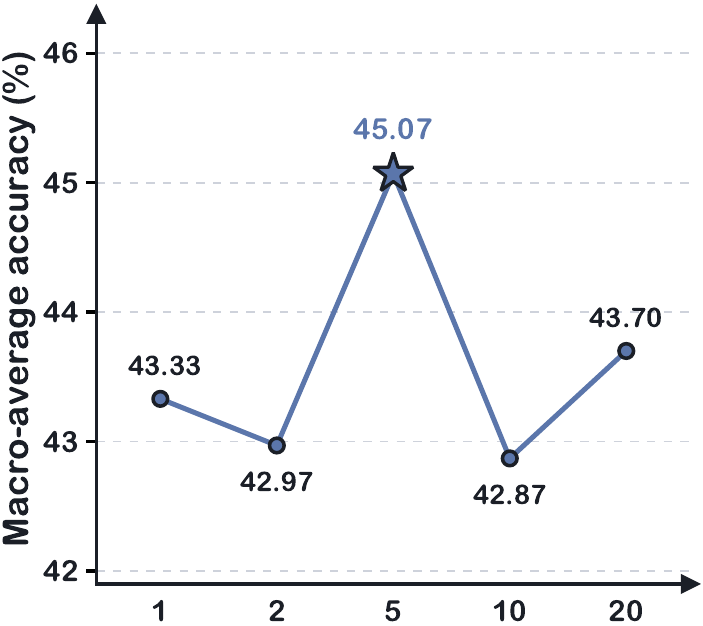}
        \caption{Propagation sensitivity $\kappa$.}
        \label{fig:prop_kappa}
    \end{subfigure}
    \hfill
    \begin{subfigure}[t]{0.4\linewidth}
        \centering
        \includegraphics[width=\linewidth]
        {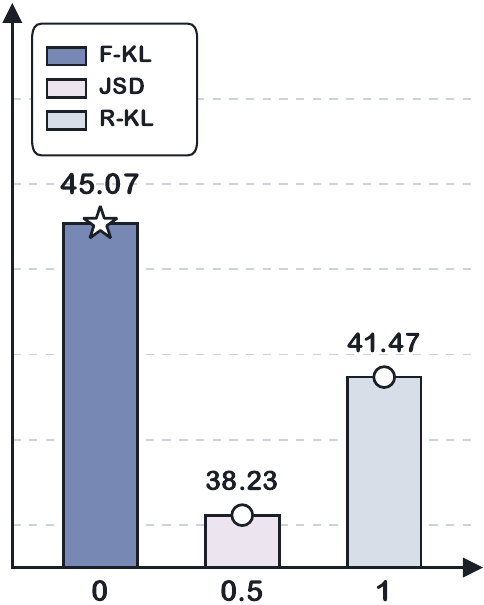}
        \caption{Local divergence.}
        \label{fig:distillation_divergence}
    \end{subfigure}

    \caption{Parameter ablations of DASH on Qwen3-1.7B.
    (a) Macro-average accuracy under different propagation
    sensitivities $\kappa$. (b) Comparison of forward KL, symmetric
    JSD, and reverse KL as the local distillation objective. Stars
    indicate the settings used in the main experiments. Values are
    means over training seeds $0,1,2,3$.}
    \label{fig:parameter_ablations}
\end{figure}

\paragraph{Propagation sensitivity $\kappa$.}
The sensitivity $\kappa$ controls the sharpness of the propagation
gates. Smaller values make the gates less responsive to
sequence-relative gaps, while larger values produce sharper changes
around the sequence mean. We evaluate
$\kappa\in\{1,2,5,10,20\}$. As shown in
Figure~\ref{fig:prop_kappa}, every tested value outperforms the standard
OPSD macro-average of 41.87, indicating that the improvement is not
limited to a single value of $\kappa$. Among the tested settings,
$\kappa=5$ achieves the highest macro-average of 45.07. The lower
performance at both smaller and larger values suggests that moderate
gate sensitivity better balances insufficient adaptation and overly
sharp responses to local gap variations. We therefore use $\kappa=5$
in the main experiments.

\paragraph{Choice of distillation divergence.}
We examine the effect of the distillation divergence by replacing the
forward KL used in the main configuration with symmetric JSD and
reverse KL~\cite{agarwal2024onpolicy,gu2024minillm,ko2024distillm}, while keeping all other settings fixed. As shown in Figure~\ref{fig:distillation_divergence}, forward KL achieves the highest macro-average score of 45.07, outperforming symmetric JSD and reverse KL by 6.84 and 3.60 points, respectively. It also performs best on all three benchmarks. These results show that the divergence choice substantially affects DASH and support the use of forward KL in the main configuration.

\paragraph{Vocabulary support.}
The main configuration computes the local forward-KL signal over the
full vocabulary. We examine whether this distributional support can be
compressed by retaining the teacher's top-$k$ tokens and merging the
remaining vocabulary into a single tail term, while keeping all other
settings unchanged.

\begin{table}[t]
\centering
\small
\setlength{\tabcolsep}{2.7pt}
\begin{tabular}{@{}lcccc@{}}
\toprule
\textbf{Variant}
& \shortstack{\textbf{AIME}\\\textbf{2024}}
& \shortstack{\textbf{AIME}\\\textbf{2025}}
& \shortstack{\textbf{HMMT}\\\textbf{2025}}
& \textbf{Avg.} \\
\midrule
OPSD, full vocab.
& 55.60 & 40.80 & 29.20 & 41.87 \\
\midrule
DASH, full vocab.
& \textbf{58.30}
& \textbf{45.80}
& \textbf{31.10}
& \textbf{45.07} \\
DASH, top-100 $+$ tail
& \underline{58.10}
& \underline{45.00}
& \underline{30.00}
& \underline{44.37} \\
DASH, top-1 $+$ tail
& 48.10 & 33.60 & 22.80 & 34.83 \\
\bottomrule
\end{tabular}
\caption{Effect of vocabulary support on Qwen3-1.7B. The top-$k$ plus
tail variants retain the teacher's top-$k$ tokens and merge the
remaining vocabulary into a single tail term. Values are four-seed
means.}
\label{tab:vocabulary_support}
\end{table}

As shown in Table~\ref{tab:vocabulary_support}, retaining the top 100
teacher tokens yields an average score of 44.37, only 0.70 points below
the full-vocabulary configuration and 2.50 points above standard OPSD.
In contrast, reducing the explicit support to only the top teacher
token lowers the average score to 34.83, a drop of 10.24 points from
the main configuration. These results indicate that a moderately
compressed support preserves most of the benefit of DASH, whereas an
overly concentrated approximation discards important distributional
information. We therefore use the full vocabulary in the main
experiments.


\section{Conclusion}

In this work, we revisit how dense token-level supervision should be allocated in on-policy self-distillation. Vanilla OPSD assigns uniform weights to local distillation signals, overlooking the temporal evolution of teacher--student discrepancies along each trajectory. To address this limitation, we propose Divergence-Adaptive Supervision Horizons (DASH), a sequence-aware method for adaptive supervision allocation. DASH transforms sequence-relative divergence gaps into adaptive propagation gates and performs backward multi-step aggregation to construct temporally conditioned token-level weights and adaptive supervision horizons. Experiments across three mathematical reasoning benchmarks and three model scales show that DASH improves the four-seed mean of our matched vanilla OPSD reruns and obtains the highest scores among the compared results. DASH reuses the local teacher and student distributions already computed by vanilla OPSD and requires no additional teacher or student forward passes. We hope our work can offer new insights for improving dense supervision allocation in on-policy learning and long-horizon reasoning.



\bibliography{aaai2027}

\clearpage
\appendix

\lstset{
  basicstyle=\footnotesize\ttfamily,
  breaklines=true,
  columns=fullflexible,
  keepspaces=true,
  showstringspaces=false,
  numbers=none,
  literate={—}{{\textemdash}}1
}

\noindent\textbf{Overview.}
This appendix is organized into five parts.
Appendix~\ref{app:theory} provides supplementary theoretical analysis of DASH,
including the fixed-horizon gradient decomposition, the clipped local
signal, the aggregation coefficients, the effective supervision
horizon, and the implemented gradient.
Appendix~\ref{sec:reproducibility} records result provenance, baseline
instantiations, the exact checkpoint-selection rule, and implementation
details not included in the main paper, including the complete LoRA
configuration, optimizer and distributed-training settings, teacher
prompts, answer verification, and hardware.
Appendix~\ref{sec:four_seed_results} reports the four-seed OPSD and DASH
protocol and uncertainty estimates over random seeds \(0,1,2,3\).
Appendix~\ref{sec:complete_ablations} presents the complete component and design-choice ablations,
including fixed propagation coefficients, gate direction and signals,
coefficient scale, propagation sensitivity, divergence geometry, and
vocabulary support.
Appendix~\ref{sec:outcome_optimization} examines the compatibility of DASH with outcome-level
optimization through the complete GRPO grid and a cross-scale
extension. Together, these materials complement the theoretical and
empirical results in the main text with complete derivations,
reproducibility details, and additional experimental evidence.

\section{Supplementary Theoretical Analysis}
\label{app:theory}

This section expands the theoretical analysis in the main text. It
first restates and proves the fixed-horizon gradient decomposition, then
connects that result to the local signal, aggregation coefficients,
effective supervision horizons, and gradients used by DASH.

\subsection{Fixed-Horizon Construction}
\label{app:fixed_horizon}

Let $\pi_\theta$ denote the student policy. For a problem $x$ with
privileged information $z$, define the filtration
$\mathcal F_u=\sigma(x,z,y_{<u})$. Training rollouts use a deterministic
maximum completion length ($1{,}024$ tokens in the main configuration).
We therefore fix a truncation horizon $H\geq1$ and represent each
rollout as a padded sequence
$\mathbf y=(y_1,\ldots,y_H)$. If an end-of-sequence token is emitted at
position $T\leq H$, the process subsequently remains in a
parameter-independent absorbing state:
\begin{equation}
    \pi_\theta(\perp\mid s_u)=1,
    \qquad
    \nabla_\theta\log\pi_\theta(\perp\mid s_u)=0,
    \qquad u>T,
    \label{eq:supp_absorbing}
\end{equation}
where $s_u=(x,y_{<u})$. We set every local signal to zero after
termination. Conditional on $x$, the padded trajectory distribution is
\begin{equation}
    P_\theta(\mathbf y\mid x)
    =
    \prod_{u=1}^{H}\pi_\theta(y_u\mid s_u).
    \label{eq:supp_padded_factorization}
\end{equation}

The main analysis considers the fixed-horizon surrogate
\begin{equation}
    \mathcal J_H(\theta)
    =
    \mathbb E_{(x,z)\sim\mathcal D}
    \mathbb E_{\mathbf y\sim P_\theta(\cdot\mid x)}
    \left[
        \frac{1}{H}\sum_{t=1}^{H}d_t
    \right].
    \label{eq:supp_fixed_horizon_objective}
\end{equation}
For fixed $(x,z)$, the padded trajectory space is finite, so the
gradient and the inner trajectory expectation commute term by term.
Moreover, $1/H$ is deterministic and can be moved outside that
expectation. These are the two properties needed below; they would not
follow from treating the realized length $T$ as a constant outside an
expectation over variable-length trajectories.

\subsection{Privileged Teacher}
\label{app:teacher}

\paragraph{Adapter-disabled privileged teacher.}
The privileged teacher is the frozen base network evaluated with the
LoRA adapter disabled. Its parameters $\bar\theta$ therefore remain
independent of the trainable adapter parameters $\theta$. Thus,
\begin{equation}
    \pi_t^{\mathrm T}
    =
    \operatorname{sg}\!\left[
        \pi_{\bar\theta}(\cdot\mid s_t,z)
    \right],
    \qquad
    \nabla_\theta\pi_t^{\mathrm T}=0.
\end{equation}

Under this frozen-teacher setting, the local forward KL
\begin{equation}
    d_t
    =
    D_{\mathrm{KL}}\!\left(
        \pi_t^{\mathrm T}
        \,\middle\|\,
        \pi_\theta(\cdot\mid s_t)
    \right)
\end{equation}
depends explicitly on $\theta$ only through the student conditional.
Accordingly, $\nabla_\theta^{\mathrm{loc}}d_t$ denotes differentiation
through that conditional at the fixed sampled prefix, with both the
teacher target and the discrete sampling path detached.

\subsection{Decomposition and Proof}
\label{app:decomposition}

\begin{proposition}[Fixed-Horizon Gradient Decomposition]
\label{prop:supp_fixed_horizon}
Let $\rho_t$ be an $\mathcal F_t$-measurable per-position scalar whose
explicit dependence on $\theta$ at a fixed sampled prefix is
differentiable almost surely, and set $\rho_t=0$ for $t>T$. Define
\begin{equation}
    \mathcal J_H^\rho(\theta)
    =
    \mathbb E\!\left[
        \frac{1}{H}\sum_{t=1}^{H}\rho_t
    \right].
\end{equation}
Then
\begin{equation}
\begin{split}
    \nabla_\theta\mathcal J_H^\rho(\theta)
    =
    \mathbb E\Bigg[
        &\frac{1}{H}\sum_{t=1}^{H}
        \nabla_\theta^{\mathrm{loc}}\rho_t\\
        &+
        \frac{1}{H}\sum_{u=1}^{H}
        G^\rho_{u+1}
        \nabla_\theta\log\pi_\theta(y_u\mid s_u)
    \Bigg],
    \label{eq:supp_fixed_horizon_decomposition}
\end{split}
\end{equation}
where
$G^\rho_{u+1}:=\sum_{t=u+1}^{H}\rho_t$ is the future
signal-to-go after position $u$.
\end{proposition}

\begin{proof}
Let $R(\mathbf y;\theta)=\sum_{t=1}^{H}\rho_t$. The finite-horizon
construction and the log-derivative identity give
\begin{equation}
\begin{split}
    \nabla_\theta\mathcal J_H^\rho(\theta)
    =
    \frac{1}{H}\,
    \mathbb E\Bigg[
        &\sum_{t=1}^{H}\nabla_\theta^{\mathrm{loc}}\rho_t\\
        &+
        R(\mathbf y;\theta)
        \nabla_\theta\log P_\theta(\mathbf y\mid x)
    \Bigg].
    \label{eq:supp_log_derivative}
\end{split}
\end{equation}
By Eq.~\eqref{eq:supp_padded_factorization},
\begin{equation}
    \nabla_\theta\log P_\theta(\mathbf y\mid x)
    =
    \sum_{u=1}^{H}
    \nabla_\theta\log\pi_\theta(y_u\mid s_u).
\end{equation}
The trajectory-dependent component is therefore
\begin{equation}
    \sum_{u=1}^{H}\sum_{t=1}^{H}
    \rho_t\,
    \nabla_\theta\log\pi_\theta(y_u\mid s_u).
    \label{eq:supp_double_sum}
\end{equation}

Fix $t\leq u$. Because $\rho_t$ is $\mathcal F_t$-measurable and
$\mathcal F_t\subseteq\mathcal F_u$, it can be taken outside the
conditional expectation over $y_u$. The score has zero conditional
mean even after conditioning on the privileged information $z$, since
$y_u$ is sampled from $\pi_\theta(\cdot\mid s_u)$. Hence, by the tower
property,
\begin{equation}
\begin{split}
    &\mathbb E\!\left[
        \rho_t
        \nabla_\theta\log\pi_\theta(y_u\mid s_u)
    \right]\\
    &\quad=
    \mathbb E\!\left[
        \rho_t\,
        \mathbb E\!\left[
            \nabla_\theta\log\pi_\theta(y_u\mid s_u)
            \,\middle|\,
            \mathcal F_u
        \right]
    \right]
    =0.
    \label{eq:supp_score_cancellation}
\end{split}
\end{equation}
Only terms with $t>u$ remain in
Eq.~\eqref{eq:supp_double_sum}; their inner sum is
$G^\rho_{u+1}$. Substitution into
Eq.~\eqref{eq:supp_log_derivative} proves
Eq.~\eqref{eq:supp_fixed_horizon_decomposition}.
\end{proof}

\paragraph{Specialization and padded positions.}
Taking $\rho_t=d_t$ recovers Proposition~\ref{prop:exact_gradient_decomposition}. The
causal cancellation above also applies to the clipped signal $r_t$
defined below because it is determined by $(s_t,z)$ before $y_t$ is
sampled. For padded positions, both the local signal and the score are
zero by construction, so the result is unaffected by the amount of
padding between $T$ and $H$.

\subsection{Clipped Local Signal}
\label{app:clipped_signal}

Let
\begin{equation}
    \ell_{t,v}
    =
    \pi_t^{\mathrm T}(v)
    \log
    \frac{\pi_t^{\mathrm T}(v)}
         {\pi_t^{\mathrm S}(v)}
\end{equation}
be the vocabulary-level summand of the forward KL, so that
$d_t=\sum_{v\in\mathcal A}\ell_{t,v}$. DASH applies the same
pointwise upper clipping used in the main text:
\begin{equation}
    r_t
    =
    \sum_{v\in\mathcal A}
    \min(\ell_{t,v},\tau).
    \label{eq:supp_clipped_signal}
\end{equation}
Thus, $r_t\leq d_t$ and $r_t\leq\tau|\mathcal A|$. Because individual
forward-KL summands can be negative, the upper-clipped sum is not
guaranteed to be non-negative. This does not affect its use as the
centered gate signal or as a local differentiable loss: $r_t$ is
$\mathcal F_t$-measurable, and the usual automatic-differentiation
subgradient is used at the clipping threshold.

\subsection{Aggregation Coefficients and Supervision Horizons}
\label{app:coefficients}

The DASH recursion
\begin{equation}
    A_T=r_T,
    \qquad
    A_t=r_t+\lambda_tA_{t+1}
\end{equation}
unrolls as
\begin{equation}
    A_t
    =
    \sum_{k=t}^{T}
    \left(
        \prod_{j=t}^{k-1}\lambda_j
    \right)r_k,
    \label{eq:supp_unrolled_aggregate}
\end{equation}
where an empty product equals one. Summing over all starting positions
and exchanging the order of summation gives
\begin{equation}
    \mathcal L_{\mathrm{DASH}}
    =
    \frac{1}{T}\sum_{k=1}^{T}c_kr_k,
    \qquad
    c_k
    =
    \sum_{t=1}^{k}
    \prod_{j=t}^{k-1}\lambda_j.
    \label{eq:supp_coefficient_closed_form}
\end{equation}
Separating the $t=k$ term yields the recursion stated in the main
paper:
\begin{equation}
    c_1=1,
    \qquad
    c_k=1+\lambda_{k-1}c_{k-1}.
\end{equation}
Since each sigmoid gate satisfies $0<\lambda_j<1$,
\begin{equation}
    1\leq c_k\leq k.
\end{equation}
For a constant gate $\lambda$, the exact finite-length expression is
\begin{equation}
    c_k=\frac{1-\lambda^k}{1-\lambda}.
\end{equation}

The effective supervision horizon beginning at position $t$ is the
total propagation mass assigned to the current and later positions:
\begin{equation}
    h_t
    =
    \sum_{k=t}^{T}
    \prod_{j=t}^{k-1}\lambda_j.
    \label{eq:supp_effective_horizon}
\end{equation}
For a constant gate,
\begin{equation}
    h_t
    =
    \frac{1-\lambda^{T-t+1}}{1-\lambda},
\end{equation}
which approaches $1/(1-\lambda)$ as the remaining rollout length grows.
Under DASH, $\lambda_t=
\operatorname{sg}[\sigma(-\kappa(r_t-\bar r))]$, so a below-average
signal leaves the boundary after position $t$ more open, whereas an
above-average signal restricts propagation across that boundary. The
coefficient and horizon views satisfy
\begin{equation}
    \sum_{k=1}^{T}c_k
    =
    \sum_{t=1}^{T}h_t,
\end{equation}
because both sides sum the same gate products over all pairs
$1\leq t\leq k\leq T$.

\subsection{Gradient of the Implemented DASH Objective}
\label{app:dash_gradient}

The gate construction is detached:
\begin{equation}
    \lambda_t
    =
    \operatorname{sg}\!\left[
        \sigma\!\left(-\kappa(r_t-\bar r)\right)
    \right],
    \qquad
    \bar r=\frac{1}{T}\sum_{t=1}^{T}r_t.
\end{equation}
The sequence mean and all subsequent operations used to construct
$\lambda_t$ and $c_k$ therefore carry no gradient. Conditional on a
sampled rollout, the implemented backward pass is
\begin{equation}
    \nabla_\theta\mathcal L_{\mathrm{DASH}}
    =
    \frac{1}{T}\sum_{k=1}^{T}
    c_k\,\nabla_\theta^{\mathrm{loc}}r_k.
    \label{eq:supp_dash_gradient}
\end{equation}
Thus, DASH performs sequence-conditioned coefficient allocation over
direct local-distillation gradients. The mean $\bar r$ couples
positions when the detached coefficients are constructed, but it does
not introduce cross-position gradients through the gates.

\subsection{Relation to the Implemented Length Normalization}
\label{app:length_normalization}

Proposition~\ref{prop:supp_fixed_horizon} deliberately concerns the
fixed-horizon surrogate $\mathcal J_H^\rho$, whereas the implemented
losses normalize each sampled rollout by its realized length $T$. For
a fixed rollout, replacing $1/H$ with $1/T$ rescales all local
coefficients by the same positive scalar. It therefore preserves both
the uniform within-rollout coefficient profile of vanilla OPSD and the
relative DASH profile $(c_1,\ldots,c_T)$. We do not claim that the
gradient of the variable-length expected objective is identical to
Eq.~\eqref{eq:supp_fixed_horizon_decomposition}.

\subsection{Scope of the Decomposition}
\label{app:scope}

The second term in
Eq.~\eqref{eq:supp_fixed_horizon_decomposition} is a score-function
trajectory term weighted by a sum of subsequent local signals. DASH
does not estimate this term, introduce score-function gradients, or
perform future-to-past credit assignment. The coefficient families are
also structurally different: $G^\rho_{u+1}$ aggregates signals after
position $u$ and multiplies the score term, whereas $c_k$ is determined
by the preceding gate path and multiplies the direct local-distillation
gradient at position $k$. The decomposition is used only to motivate
order-dependent coefficient allocation within the differentiable local
distillation channel; it does not imply that DASH recovers the omitted
trajectory gradient.

\section{Additional Reproducibility Details}
\label{sec:reproducibility}

This section makes explicit which results are externally reported,
which are produced by our reruns, how each comparison method is
instantiated, and how checkpoints are selected. It then records the
remaining implementation details recovered from the saved run
artifacts and launch scripts.

\subsection{Result Provenance and Comparison Scope}
\label{sec:supp_provenance}

Table~\ref{tab:supp_result_provenance} separates external reference
results from our reruns. Base, SFT, and GRPO are taken from Zhao et
al.~\cite{zhao2026selfdistilled} and are marked with $\dagger$ in
Table~\ref{tab:main_results}. These external values are not subjected
to our 200-step budget or checkpoint-selection rule. All remaining
entries in Table~\ref{tab:main_results} are produced by us.

\begin{table}[!htbp]
\centering
\scriptsize
\setlength{\tabcolsep}{3.5pt}
\begin{tabular}{@{}lll@{}}
\toprule
\textbf{Method} & \textbf{Source} & \textbf{Training runs} \\
\midrule
Base & Zhao et al. & Not applicable \\
SFT, GRPO & Zhao et al. & As reported \\
OPSD & Our rerun & $s\in\{0,1,2,3\}$ \\
EOPD & Our rerun & One ($s=0$) \\
AVSD & Our rerun & One ($s=0$) \\
PW-OPSD & Our rerun & One ($s=0$) \\
DASH & Our rerun & $s\in\{0,1,2,3\}$ \\
DASH ablations & Our reruns & $s\in\{0,1,2,3\}$ \\
\bottomrule
\end{tabular}
\caption{Sources and training repetitions of the compared results.}
\label{tab:supp_result_provenance}
\end{table}

The four-seed OPSD and DASH means reduce dependence on any single
training initialization. They are not interpreted as statistical significance
tests against EOPD, AVSD, or PW-OPSD, which are single-run reruns, or
against the external Base, SFT, and GRPO results. Accordingly, claims
involving these methods refer to the compared or displayed results
rather than to statistically significant differences.

\subsection{Baseline Instantiations}
\label{sec:supp_baseline_instantiations}

For OPSD, EOPD, PW-OPSD, and DASH, the student receives only the
problem and generates its own on-policy rollout. The privileged teacher
uses the same base network with the LoRA adapter disabled, additionally
conditions on the reference solution (including the final answer), and
scores the student-generated tokens by teacher forcing on that rollout.
These methods share the same base checkpoint, data preprocessing,
student rollout configuration, and privileged-teacher context in our
reruns, while retaining their method-specific objectives.

OPSD uniformly averages the local forward-KL signals. EOPD instantiates
its entropy-aware objective on this matched privileged-OPSD backbone;
it therefore isolates the transfer of the EOPD objective rather than
reproducing the external-teacher architecture used in the original
paper. PW-OPSD retains its predefined position-dependent weights.
DASH replaces uniform or predefined weighting with the proposed
sequence-conditioned multi-step aggregation.

AVSD retains its method-defining multi-view teacher construction rather
than using the single reference-conditioned teacher above. Its teacher
views are constructed from the full solution, a partial solution, and
the final answer. Apart from this method-specific construction and
objective, our AVSD run uses the same base model, training data,
200-step outer budget, checkpoint-selection rule, and final evaluation
procedure as the other reruns.

\subsection{Training Budget and Checkpoint Selection}
\label{sec:supp_checkpoint_selection}

Every method rerun by us is trained for the complete 200 optimization
steps. We save checkpoints every 20 steps, giving the candidate set
\begin{equation}
    \mathcal C=\{20,40,\ldots,200\}.
\end{equation}
Fix a method and model scale. For training seed $s$, candidate
checkpoint $t\in\mathcal C$, and benchmark $b$, let $m_{s,t,b}$ denote
the corresponding Avg@12 score. We define the selection score
\begin{equation}
    a_{s,t}
    =
    \frac{1}{3}
    \sum_{b\in\mathcal B}m_{s,t,b},
    \qquad
    \mathcal B=\{\mathrm{AIME24},\mathrm{AIME25},\mathrm{HMMT25}\},
    \label{eq:supp_checkpoint_score}
\end{equation}
and select
\begin{equation}
    t_s^\star
    =
    \arg\max_{t\in\mathcal C}a_{s,t}.
    \label{eq:supp_checkpoint_selection}
\end{equation}
All three benchmark scores for seed $s$ come from this single selected
checkpoint:
\begin{equation}
    m_{s,b}=m_{s,t_s^\star,b}.
    \label{eq:supp_selected_benchmark_score}
\end{equation}
For a single-run rerun, $s=0$ and the selected scores are reported
directly without averaging across training runs. Because
$\mathcal B$ is also the final benchmark set, this is a
best-within-200-step reporting protocol, not held-out validation
selection.

\paragraph{Adapters and optimization.}
For the matched privileged-OPSD runs, the adapter configuration is
identical at all three model scales.
LoRA is applied to the attention projections \texttt{q\_proj},
\texttt{k\_proj}, \texttt{v\_proj}, and \texttt{o\_proj}, and to the
MLP projections \texttt{gate\_proj}, \texttt{up\_proj}, and
\texttt{down\_proj}. The effective scaling is $\alpha/r=2.0$ and the
adapter dropout is 0.05. For causal language modeling, bias,
rank-stabilized LoRA, DoRA, and additional trainable modules are
disabled. The $A$ matrix uses Kaiming
initialization and $B$ is initialized to zero. The base model,
embeddings, and language-model head remain frozen.

We use fused AdamW with $\beta_1=0.9$, $\beta_2=0.999$,
$\epsilon=10^{-8}$, and zero weight decay. The learning rate decays
linearly from the value stated in the main text to zero over the fixed
optimization horizon, without warmup. Although the trainer stores
\texttt{num\_train\_epochs=30}, the step limit is reached first and
therefore determines the training duration. Training uses bfloat16
mixed precision, DeepSpeed ZeRO-2 with optimizer-state offloading,
gradient checkpointing, and FlashAttention-2. The maximum combined
sequence length is 20,000 tokens. The privileged-teacher forward pass
is evaluated with the adapter disabled, so only the student LoRA path
receives parameter updates.

\paragraph{Batch decomposition and runtime.}
Table~\ref{tab:supp_batch_runtime} combines the scale-specific DASH batch
decomposition and runtime information that previously required two
separate tables. All three configurations realize the shared effective
batch size of 64.

\begin{table}[!htbp]
\centering
\footnotesize
\setlength{\tabcolsep}{3.5pt}
\begin{tabular}{@{}lcc@{}}
\toprule
\textbf{Model} & \textbf{Micro/accum./GPU} & \textbf{Hardware; time} \\
\midrule
Qwen3-1.7B & $4/2/8$  & $8\times$ A800; $\sim$45 min \\
Qwen3-4B   & $2/4/8$  & $8\times$ A800; $\sim$2.3 h \\
Qwen3-8B   & $1/16/4$ & $4\times$ A800; $\sim$15 h \\
\bottomrule
\end{tabular}
\caption{DASH batch decomposition and approximate runtime. The middle
column reports per-device batch size, gradient accumulation, and GPU
count; runtime includes scheduler and queue overhead.}
\label{tab:supp_batch_runtime}
\label{tab:supp_batch_decomposition}
\label{tab:supp_runtime}
\end{table}

All DASH runs use NVIDIA A800-SXM4-80GB GPUs. On-policy rollout generation
dominates runtime. Relative to vanilla OPSD, DASH requires no additional
teacher or student forward pass; its extra scalar backward scan accounts
for less than 1\% of the step time in matched run timestamps.

\paragraph{Matched OPSD-family prompts.}
For OPSD, EOPD, PW-OPSD, and DASH, the student receives only the
problem, rendered with the Qwen3 chat
template using \texttt{enable\_thinking=True} and
\texttt{add\_generation\_prompt=True}; no additional instruction is
appended. It then generates its own on-policy rollout. The privileged
teacher uses the same base network with the LoRA adapter disabled. It
receives the problem and reference solution, including the final
answer, followed by the two instruction strings below, and scores the
student-generated tokens by teacher forcing on the same rollout.

\paragraph{Reference-analysis instruction.}
\begin{lstlisting}
The reference reasoning above arrives at the correct answer. Please analyze this solution and explain the key reasoning steps and problem-solving strategies employed. Do NOT use <think> tags. Do NOT derive your own solution. Simply analyze and explain the reference solution provided above.
\end{lstlisting}

\paragraph{Transition to independent derivation.}
\begin{lstlisting}
After reading the reference solution above, make sure you truly understand the reasoning behind each step — do not copy or paraphrase it. Now, using your own words and independent reasoning, derive the same final answer to the problem above. Think step by step, explore different approaches, and don't be afraid to backtrack or reconsider if something doesn't work out:
\end{lstlisting}

In these matched OPSD-family runs, the privileged reference is never
provided to the student during
rollout generation or benchmark evaluation. Training sequences use
right padding.

\paragraph{Answer extraction and equivalence checking.}
The reported benchmark scores use a symbolic verifier. For each model
response, the evaluator extracts the last occurrence of
\verb|\boxed{...}|. If the extracted prediction contains no math
delimiters, it is wrapped in dollar signs. The prediction and ground
truth are parsed with
\texttt{math\_verify.parse(..., fallback\_mode="no\_fallback")}, and
equivalence is checked with a five-second timeout using
\texttt{math\_verify.verify}. If parsing or verification raises an
exception, the evaluator falls back to exact comparison after removing
dollar signs and spaces and converting both strings to lowercase.

The optional outcome-level extension uses a separate reward-time
grader. It extracts the last boxed answer, applies the EleutherAI-style
\texttt{strip\_string} normalization, and checks exact string equality.
The normalization canonicalizes common forms involving fractions,
square roots, delimiters, units, percentages, degree symbols, spaces,
and simple decimal or division expressions. This grader is used only by
the optional outcome-level training objective, not for any reported
benchmark score.

\section{Four-Seed Results and Statistical Robustness}
\label{sec:four_seed_results}

The matched OPSD configurations, the main DASH configurations, and all
DASH ablations are trained with four random seeds,
\(s\in\{0,1,2,3\}\). For benchmark \(b\), let \(m_{s,b}\) denote the
Avg@12 score from the checkpoint selected for seed \(s\) by
Eq.~\eqref{eq:supp_checkpoint_selection}. We report the arithmetic mean
and sample standard deviation
\begin{equation}
    \mu_b=\frac{1}{4}\sum_{s=0}^{3}m_{s,b},
    \qquad
    \operatorname{Std}_b
    =
    \sqrt{
        \frac{1}{3}
        \sum_{s=0}^{3}(m_{s,b}-\mu_b)^2
    }.
    \label{eq:supp_seed_statistics}
\end{equation}
The four training seeds are distinct from the 12 independently sampled
responses used to compute Avg@12 for each problem.

The OPSD and DASH entries in Table~\ref{tab:main_results} are four-seed
means. In particular, the Qwen3-1.7B OPSD mean
$(55.60,40.80,29.20)$, with macro-average $41.87$, is our matched rerun
and serves as the common OPSD anchor in all Qwen3-1.7B ablations. It is
not the OPSD value reported by Zhao et al.~\cite{zhao2026selfdistilled}.

\begin{table}[!htbp]
\centering
\footnotesize
\setlength{\tabcolsep}{3.7pt}
\begin{tabular}{@{}lccc@{}}
\toprule
\textbf{Benchmark} & \textbf{1.7B} & \textbf{4B} & \textbf{8B} \\
\midrule
AIME 2024 & $58.30\pm0.90$ & $77.20\pm0.70$ & $78.90\pm0.90$ \\
AIME 2025 & $45.80\pm1.20$ & $71.10\pm0.80$ & $71.40\pm1.00$ \\
HMMT 2025 & $31.10\pm0.30$ & $46.70\pm0.70$ & $48.90\pm1.30$ \\
\midrule
Macro avg. & 45.07 & 65.00 & 66.40 \\
\bottomrule
\end{tabular}
\caption{Four-seed DASH results. Benchmark rows report mean $\pm$
sample standard deviation over seeds $0,1,2,3$ ($n=4$); the last row is
the arithmetic mean of the three benchmark means.}
\label{tab:supp_four_seed_results}
\end{table}

The benchmark-level standard deviations range from 0.30 to 1.30 points
across the three model scales. The means in
Table~\ref{tab:supp_four_seed_results} are those reported in the main
paper. We do not derive a standard deviation for the macro-average from
the three marginal standard deviations because the benchmark scores are
paired within each training seed. The four-seed means reduce dependence
on any single training initialization; they are not used to claim
statistical significance against the single-run or externally reported
baselines.

\section{Complete Ablation Results}
\label{sec:complete_ablations}

This section expands the ablations summarized in the main text with
complete benchmark-level results and the additional gate-signal study.
Unless otherwise specified, all variants use Qwen3-1.7B and the same
training, checkpoint-selection, and evaluation protocol. Each entry,
including the OPSD configuration obtained with $\lambda=0$, is the mean
Avg@12 score over training seeds \(0,1,2,3\); the final column is the
unweighted mean of the three benchmark scores. Thus, the OPSD reference
of 41.87 is our matched four-seed rerun rather than an external result.
A24, A25, and H25 abbreviate AIME 2024, AIME 2025, and HMMT February
2025.

\paragraph{Propagation coefficient and gate direction.}
To separate multi-step aggregation from sequence-adaptive propagation,
we replace the token-wise gates with a constant \(\lambda\) shared by
all positions and rollouts; \(\lambda=0\) recovers vanilla OPSD. The
inverse-gap control instead preserves the sequence-relative gap,
sigmoid gate, sensitivity, and backward recursion but reverses the sign
of the gate input:
\begin{equation}
    \lambda_t^{\mathrm{Inv}}
    =
    \operatorname{sg}\!\left[
        \sigma\!\left(+\kappa(r_t-\bar r)\right)
    \right].
\end{equation}

\begin{table}[!htbp]
\centering
\footnotesize
\setlength{\tabcolsep}{2.8pt}
\begin{tabular}{@{}lcccc@{}}
\toprule
\textbf{Setting} & \textbf{A24} & \textbf{A25} & \textbf{H25} & \textbf{Avg.} \\
\midrule
OPSD ($\lambda=0$) & 55.60 & 40.80 & 29.20 & 41.87 \\
Fixed $\lambda=0.1$ & 58.10 & 44.20 & 28.60 & 43.63 \\
Fixed $\lambda=0.3$ & 58.10 & 42.20 & 29.40 & 43.23 \\
Fixed $\lambda=0.5$ & 55.60 & 42.20 & 30.30 & 42.70 \\
Fixed $\lambda=0.7$ & 56.40 & 41.90 & 29.40 & 42.57 \\
Fixed $\lambda=0.9$ & 56.70 & 41.90 & 28.90 & 42.50 \\
Inverse-gap & 56.90 & 41.90 & 27.50 & 42.10 \\
\textbf{DASH} & \textbf{58.30} & \textbf{45.80} & \textbf{31.10} & \textbf{45.07} \\
\bottomrule
\end{tabular}
\caption{Four-seed means for fixed-coefficient and gate-direction
controls; $\lambda=0$ is the matched OPSD reference.}
\label{tab:supp_propagation_controls}
\label{tab:supp_fixed_lambda}
\label{tab:supp_gate_direction}
\end{table}

Every fixed coefficient improves over OPSD in macro-average;
\(\lambda=0.1\) is strongest among them at 43.63. Inverse-gap reaches
42.10, whereas DASH reaches 45.07, supporting sequence-conditioned
propagation and the selected gap-to-gate direction.

\paragraph{Coefficient allocation, average scale, and gate signal.}
Table~\ref{tab:supp_allocation_signal} separates relative coefficient
allocation from average scale and then varies the propagation-gate
signal. Soft-OR smoothly combines the entropy and divergence-gap gates;
the OPSD and DASH rows in the first panel serve as shared anchors.

\begin{table}[!htbp]
\centering
\footnotesize
\setlength{\tabcolsep}{2.5pt}
\begin{tabular}{@{}lcccc@{}}
\toprule
\textbf{Variant} & \textbf{A24} & \textbf{A25} & \textbf{H25} & \textbf{Avg.} \\
\midrule
\multicolumn{5}{@{}l}{\textit{Coefficient allocation and scale}} \\
OPSD & 55.60 & 40.80 & 29.20 & 41.87 \\
Scale-matched OPSD & 56.40 & 41.90 & 29.40 & 42.57 \\
Normalized DASH & 57.80 & 44.40 & 30.60 & 44.27 \\
DASH & 58.30 & 45.80 & 31.10 & 45.07 \\
\addlinespace[2pt]
\multicolumn{5}{@{}l}{\textit{Propagation-gate signal}} \\
Entropy ($\kappa=1$) & \textbf{58.60} & 45.60 & 28.90 & 44.37 \\
Entropy ($\kappa=2$) & 55.80 & 42.50 & 30.30 & 42.87 \\
Entropy ($\kappa=5$) & 54.70 & 43.90 & 27.50 & 42.03 \\
Soft-OR ($\kappa=5$) & 56.90 & 41.90 & \textbf{32.50} & 43.77 \\
\bottomrule
\end{tabular}
\caption{Four-seed means for coefficient-allocation controls and
alternative gate signals. OPSD and DASH are the shared anchors.}
\label{tab:supp_allocation_signal}
\label{tab:supp_scale_allocation}
\label{tab:supp_gate_signals}
\end{table}

Dynamic allocation improves macro-average by 2.40 points at the OPSD
scale and 2.50 at the DASH scale; increasing scale alone contributes
0.70 and 0.80 points. Entropy and Soft-OR gates also improve over OPSD
for suitable settings, while the divergence-gap anchor is strongest.
Entropy changes from 44.37 at \(\kappa=1\) to 42.03 at \(\kappa=5\).

\paragraph{Sensitivity, divergence geometry, and vocabulary support.}
Table~\ref{tab:supp_local_designs} collects the remaining choices. The
OPSD and DASH rows are shared anchors; tail rows merge teacher mass
outside the retained top-$k$ vocabulary.

\begin{table}[!htbp]
\centering
\footnotesize
\setlength{\tabcolsep}{2.6pt}
\begin{tabular}{@{}lcccc@{}}
\toprule
\textbf{Setting} & \textbf{A24} & \textbf{A25} & \textbf{H25} & \textbf{Avg.} \\
\midrule
\multicolumn{5}{@{}l}{\textit{Shared anchors}} \\
OPSD & 55.60 & 40.80 & 29.20 & 41.87 \\
\textbf{DASH} & \textbf{58.30} & \textbf{45.80} & \textbf{31.10} & \textbf{45.07} \\
\addlinespace[2pt]
\multicolumn{5}{@{}l}{\textit{Propagation sensitivity (alternatives to $\kappa=5$)}} \\
$\kappa=1$  & \textbf{58.60} & 42.20 & 29.20 & 43.33 \\
$\kappa=2$  & 57.50 & 43.10 & 28.30 & 42.97 \\
$\kappa=10$ & 55.80 & 43.60 & 29.20 & 42.87 \\
$\kappa=20$ & \textbf{58.60} & 43.10 & 29.40 & 43.70 \\
\addlinespace[2pt]
\multicolumn{5}{@{}l}{\textit{Local divergence (alternatives to forward KL)}} \\
Symmetric JSD & 51.90 & 38.10 & 24.70 & 38.23 \\
Reverse KL & 54.70 & 41.10 & 28.60 & 41.47 \\
\addlinespace[2pt]
\multicolumn{5}{@{}l}{\textit{Vocabulary support (alternatives to full vocabulary)}} \\
Top-100 $+$ tail & 58.10 & 45.00 & 30.00 & 44.37 \\
Top-1 $+$ tail & 48.10 & 33.60 & 22.80 & 34.83 \\
\bottomrule
\end{tabular}
\caption{Four-seed means for propagation sensitivity, divergence, and
vocabulary support. DASH denotes $\kappa=5$, forward KL, and full
vocabulary.}
\label{tab:supp_local_designs}
\label{tab:supp_kappa}
\label{tab:supp_divergence}
\label{tab:supp_vocabulary}
\end{table}

All sensitivities exceed OPSD, with \(\kappa=5\) strongest at 45.07.
Forward KL exceeds symmetric JSD and reverse KL by 6.84 and 3.60
macro-average points. Top-100 support is only 0.70 points below the full
vocabulary and remains 2.50 above OPSD; top-1 support falls to 34.83.

\section{Compatibility with Outcome-Level RL}
\label{sec:outcome_optimization}

We optionally add a GRPO correctness term (weight \(\eta\), \(K\)
rollouts) to DASH; \(\gamma=\lambda=0.95\) aggregation affects only the
GRPO path. Results are four-seed Avg@12 means. Aggregation improves all
four matched 1.7B settings (Table~\ref{tab:supp_grpo_grid}). The strongest
hybrid, \(\eta=0.3,K=8\), obtains
$(57.80,45.00,31.70;44.83)$ on A24, A25, H25, and average, versus
$(58.30,45.80,31.10;45.07)$ for DASH. At 8B, the same hybrid reaches
$(80.00,73.30,48.60;67.30)$ versus DASH's
$(78.90,71.40,48.90;66.40)$, so compatibility is scale-dependent.

\begin{table}[!htbp]
\centering
\scriptsize
\setlength{\tabcolsep}{5pt}
\begin{tabular}{@{}ccccc@{}}
\toprule
\(\boldsymbol{\eta}\) & \(\boldsymbol{K}\) & \textbf{No agg.} & \textbf{Agg.} & \textbf{Gain} \\
\midrule
0.3 & 4 & 42.80 & 43.17 & +0.37 \\
0.3 & 8 & 43.13 & \textbf{44.83} & +1.70 \\
0.5 & 4 & 42.50 & 42.97 & +0.47 \\
0.5 & 8 & 42.50 & 44.50 & \textbf{+2.00} \\
\bottomrule
\end{tabular}
\caption{Four-seed Qwen3-1.7B macro-averages for GRPO-path aggregation.}
\label{tab:supp_outcome_combined}
\label{tab:supp_grpo_grid}
\end{table}

\end{document}